\documentclass{article}

\usepackage{graphicx} % more modern
\usepackage{subfigure}

\usepackage{natbib}

\usepackage{algorithm}
\usepackage{algorithmic}

\usepackage{graphics} % for pdf, bitmapped graphics files
\usepackage{amsmath} % assumes amsmath package installed
\usepackage{amssymb}  % assumes amsmath package installed
\usepackage{amsthm}
\usepackage{bbm}     % for indicator function

\usepackage{hyperref}

\usepackage[accepted]{icml}
\icmltitlerunning{Policy Evaluation with Variance Related Risk Criteria}

\newtheorem{theorem}{Theorem}
\newtheorem{proposition}[theorem]{Proposition}
\newtheorem{lemma}[theorem]{Lemma}
\newtheorem{definition}[theorem]{Definition}
\newtheorem{assumption}[theorem]{Assumption}
\newtheorem{condition}{A}

\newcommand{\E}{\mathbb{E}}
\newcommand{\R}{\mathbb{R}}

\begin{document}

\twocolumn[
\icmltitle{Policy Evaluation with Variance Related Risk Criteria in Markov Decision Processes}

% unless you've provided the [accepted] option to the icml2013
\icmlauthor{Aviv Tamar}{avivt@tx.technion.ac.il}
\icmlauthor{Dotan Di Castro}{dot@tx.technion.ac.il}
\icmlauthor{Shie Mannor}{shie@ee.technion.ac.il}
\icmladdress{Department of Electrical Engineering, The Technion - Israel Institute of Technology, Haifa, Israel 32000}

\icmlkeywords{Reinforcement learning, risk-sensitive optimization, mean-variance estimation}

\vskip 0.3in
]

\begin{abstract}
%Managing risk in dynamic decision problems is of cardinal importance in many fields such as finance
%and process control. Popular approaches to defining risk are through various variance related criteria such as the Sharpe Ratio or the standard deviation adjusted reward.
In this paper we extend temporal difference policy evaluation algorithms to performance criteria that include the variance of the cumulative reward. Such criteria are useful for risk management, and are important in domains such as finance
and process control. We propose both TD(0) and LSTD($\lambda$) variants with linear function approximation, prove their convergence, and demonstrate their utility in a 4-dimensional continuous state space problem.
\end{abstract}

%%%%%%%%%%%%%%%%%%%%%%%%%%%%%%%%%%%%%%%%%%%%%%%%%%%%%%%%%%%%%%%%%%%%%%%%
\section{Introduction}
%%%%%%%%%%%%%%%%%%%%%%%%%%%%%%%%%%%%%%%%%%%%%%%%%%%%%%%%%%%%%%%%%%%%%%%%

In both Reinforcement Learning (RL; \citealp{BT96}) and planning in Markov Decision Processes (MDPs; \citealp{Puterman1994}), the typical objective is to maximize the cumulative (possibly discounted) expected reward, denoted by $J$.
%When the model's parameters are known, several well-established and efficient optimization algorithms are known. When the model parameters are not known, learning is needed and there are several algorithmic frameworks that solve the learning problem efficiently, at least when the model is finite.
In many applications, however, the decision maker is also interested in minimizing some form of \emph{risk} of the policy.
By risk, we mean reward criteria that take into account not only the expected reward, but also some additional statistics of the total reward such as its variance, its Value at Risk, etc. \cite{luenbergerinvestment}.

In this work we focus on risk measures that involve the \emph{variance of the cumulative reward}, denoted by $V$.
Typical performance criteria that fall under this definition include
\renewcommand{\theenumi}{(\alph{enumi})}
\renewcommand{\labelenumi}{\theenumi}
\begin{enumerate}
\item Maximize $J$ s.t. $V \le c$
\item Minimize $V$ s.t. $J \ge c$
\item Maximize the Sharpe Ratio: $J/\sqrt{V}$
\item Maximize $J - c \sqrt{V}$
\end{enumerate}
The rationale behind our choice of risk measure is that these performance criteria, such as the Sharpe Ratio \cite{sharpe1966mutual} mentioned above, are being used in practice. Moreover, it seems that human decision makers understand how to use variance well, in comparison to exponential utility functions \cite{Howard1972Risk}, which require determining a non-intuitive exponent coefficient.

A fundamental concept in RL is the the value function - the expected reward to go from a given state. Estimates of the value function drive most RL algorithms, and efficient methods for obtaining these estimates have been a prominent area of research. In particular, Temporal Difference (TD; \cite{SutBar98}) based methods have been found suitable for problems where the state space is large, requiring some sort of function approximation. TD methods enjoy theoretical guarantees (\citealp{Ber2012DynamicProgramming}; \citealp{lazaric2010finite}) and empirical success \cite{tesauro1995temporal}, and are considered the state of the art in policy evaluation.

In this work we present a TD framework for estimating the \emph{variance of the reward to go}. Our approach is based on the following key observation: the second moment of the reward to go, denoted by $M$, together with the value function $J$, obey a linear equation - similar to the Bellman equation that drives regular TD algorithms. By extending TD methods to jointly estimate $J$ and $M$, we obtain a solution for estimating the variance, using the relation $V = M - J^2$.

%\newpage
We propose both a variant of Least Squares Temporal Difference (LSTD) \cite{boyan2002technical} and of TD(0) \cite{SutBar98} for jointly estimating $J$ and $M$ with a linear function approximation. For these algorithms, we provide convergence guarantees and error bounds. In addition, we introduce a novel approach for enforcing the approximate variance to be positive, through a constrained TD equation.

Finally, an empirical evaluation on a challenging continuous maze domain highlights both the usefulness of our approach, and the importance of the variance function in understanding the risk of a policy.

This paper is organized as follows. In Section \ref{sec:background} we present our formal RL setup. In Section \ref{sec:MDP_var_TD} we derive the fundamental equations for jointly approximating $J$ and $M$, and discuss their properties. A solution to these equations may be obtained by simulation, through the use of TD algorithms, as presented in Section \ref{sec:var_estimation}. In Section \ref{sec:positive_var} we further extend the LSTD framework by forcing the approximated variance to be positive. Section \ref{sec:experiments} presents an empirical evaluation, and Section \ref{sec:conclusion} concludes, and discusses future directions.

%%%%%%%%%%%%%%%%%%%%%%%%%%%%%%%%%%%%%%%%%%%%%%%%%%%%%%%%%%%%%%%%%%%%%%%%
\section{Framework and Background}\label{sec:background}
%%%%%%%%%%%%%%%%%%%%%%%%%%%%%%%%%%%%%%%%%%%%%%%%%%%%%%%%%%%%%%%%%%%%%%%%

We consider a Stochastic Shortest Path (SSP) problem\footnote{This is also known as an episodic setup.} \cite{Ber2012DynamicProgramming}, where the environment is modeled by an MDP in discrete time with a finite state set $X\triangleq\{1,\ldots,n\}$ and a terminal state $x^*$. A fixed policy $\pi$ determines, for each $x\in X$, a stochastic transition to a subsequent state $y \in \{ X \cup x^*\}$ with probability $P(y|x)$. We consider a deterministic and bounded reward function $r:X \to \R$.
We denote by $x_k$ the state at time $k$, where $k = 0,1,2,\ldots$.

A policy is said to be \emph{proper} \cite{Ber2012DynamicProgramming} if there is a positive probability that the terminal state $x^*$ will be reached after at most $n$ transitions, from any initial state. In this paper we make the following assumption
\begin{assumption}\label{assumption:proper_policy}
The policy $\pi$ is proper.
\end{assumption}

Let $\tau\triangleq\min\{k>0|x_{k}=x^*\}$ denote the first visit time to the terminal state, and let the random variable $B$ denote the accumulated reward along the trajectory until that time\footnote{We do not define the reward at the terminal state as it is not relevant to our performance criteria. However, the customary zero terminal reward may be assumed throughout the paper.}
\begin{equation*}
\label{eq:B}
B\triangleq\sum_{k=0}^{\tau-1}r(x_k).
\end{equation*}
In this work, we are interested in the mean-variance tradeoff in $B$, represented by the \emph{value function}
\begin{equation*}
\label{eq:value_func}
J(x)\triangleq\E\left[B | x_{0}=x\right], \quad x\in X,
\end{equation*}
and the \emph{variance of the reward to go}
\begin{equation*}
\label{eq:var_func}
V(x)\triangleq\textrm{Var}\left[B | x_{0}=x\right], \quad x\in X.
\end{equation*}
We will find it convenient to define also the \emph{second moment of the reward to go}
\begin{equation*}
\label{eq:var_func}
M(x)\triangleq \E \left[B^2 | x_{0}=x\right], \quad x\in X.
\end{equation*}

Our goal is to estimate $J(x)$ and $V(x)$ from trajectories obtained by simulating the MDP with policy $\pi$.

%%%%%%%%%%%%%%%%%%%%%%%%%%%%%%%%%%%%%%%
\section{Approximation of the Variance of the Reward To Go}\label{sec:MDP_var_TD}
%%%%%%%%%%%%%%%%%%%%%%%%%%%%%%%%%%%%%%%
In this section we derive a projected equation method for approximating $J(x)$ and $M(x)$ using linear function approximation. The estimation of $V(x)$ will then follow from the relation $V(x) = M(x) - J(x)^2$.

Our starting point is a system of equations for $J(x)$ and $M(x)$, first derived by \citet{sobel1982variance} for a discounted infinite horizon case, and extended here to the SSP case. Note that the equation for $J$ is the well known Bellman equation for a fixed policy, and independent of the equation for $M$.
\begin{proposition}\label{prop:equations}
The following equations hold for $x\in X$
\begin{equation}\label{eq:bellman_TJ}
    \begin{split}
        J(x) &= r(x) + \sum_{y\in X}P(y|x) J(y), \\
        M(x) &= r(x)^2 + 2r(x) \sum_{y\in X}P(y|x) J(y)
                + \sum_{y\in X}P(y|x) M(y).
%       V(x) = \rho(x) + \sum_{y\ne x^*}P_{\theta}(y|x) V(y)
    \end{split}
\end{equation}
Furthermore, under Assumption \ref{assumption:proper_policy} a unique solution to \eqref{eq:bellman_TJ} exists.
\end{proposition}
The proof is straightforward, and given in Appendix \ref{supp:prop:equations}.

At this point the reader may wonder why an equation for $V$ is not presented. While such an equation may be derived, as was done in \cite{Tamar2012mean_var}, it is not linear. The linearity of \eqref{eq:bellman_TJ} is the key to our approach.
As we show in the next subsection, the solution to \eqref{eq:bellman_TJ} may be expressed as the fixed point of a linear mapping in the joint space of $J$ and $M$. We will then show that a projection of this mapping onto a linear feature space is contracting, thus allowing us to use existing TD theory to derive estimation algorithms for $J$ and $M$.

%%%%%%%%%%%%%%%%%%%%%%%%%%%%%%%%%%%
\subsection{A Projected Fixed Point Equation on the Joint Space of $J$ and $M$}
%%%%%%%%%%%%%%%%%%%%%%%%%%%%%%%%%%%

For the sequel we introduce the following vector notations. We denote by $P\in \R^{n \times n}$ and $r\in \R^n$ the SSP transition matrix and reward vector, i.e., $P_{x,y} = P(y|x)$ and $r_x = r(x)$, where $x,y\in X$. Also, we define $R \triangleq diag(r)$.

For a vector $z\in \R^{2n}$ we let $z_J\in \R^n$ and $z_M\in \R^n$ denote its leading and ending $n$ components, respectively. Thus, such a vector belongs to the joint space of $J$ and $M$.

We define the mapping $T:\R^{2n}\to \R^{2n}$ by
%For any $J\in \R^n$, $M\in \R^n$, and a policy parameter $\theta$, let $T_{J}$ and $T_{M}$ denote the mappings
\begin{equation*}
\begin{split}
  [T z]_J &= r + P z_J, \\
  [T z]_M &= Rr + 2R P z_J + P z_M.
\end{split}
\end{equation*}
It may easily be verified that a fixed point of $T$ is a solution to \eqref{eq:bellman_TJ}, and by Proposition \ref{prop:equations} such a fixed point exists and is unique.
%and let $T_{JM}$ denote a mapping that is a concatenation of $T_J$ and $T_M$
%\begin{equation*}
%    T_{JM}(J;M)(x) =
%    \begin{cases} T_J(x) & \text{if $1\leq x \leq n$}
%    \\
%    T_M(x-n) & \text{if $n+1 \leq x \leq 2n$}
%    \end{cases}.
%\end{equation*}

When the state space $X$ is large, a direct solution of \eqref{eq:bellman_TJ} is not feasible, even if $P$ may be accurately obtained. A popular approach in this case is to approximate $J(x)$ by restricting it to a lower dimensional subspace, and use simulation based TD algorithms to adjust the approximation parameters \cite{Ber2012DynamicProgramming}. In this paper we extend this approach to the approximation of $M(x)$ as well.

We consider a linear approximation architecture of the form
\begin{equation}\label{linear_FA}
    \begin{split}
      \tilde{J}(x) &= \phi_J(x)^T w_J, \\
      \tilde{M}(x) &= \phi_M(x)^T w_M,
    \end{split}
\end{equation}
where $w_J\in\R^{l_J}$ and $w_M\in\R^{l_M}$ are the approximation parameter vectors, $\phi_J(x)\in\R^{l_J}$ and $\phi_M(x)\in\R^{l_M}$ are state dependent features, and $(\cdot)^T$ denotes the transpose of a vector. The low dimensional subspaces are therefore
\begin{equation*}
    \begin{split}
       S_J &= \{ \Phi_J w | w\in\R^{s_J} \}, \\
       S_M &= \{ \Phi_M w | w\in\R^{s_M} \},
     \end{split}
\end{equation*}
where $\Phi_J$ and $\Phi_M$ are matrices whose rows are $\phi_J(x)^T$ and $\phi_M(x)^T$, respectively. We make the following standard independence assumption on the features
\begin{assumption}\label{assumption:Phi_rank}
The matrix $\Phi_J$ has rank $l_J$ and the matrix $\Phi_M$ has rank $l_M$.
\end{assumption}

As outlined earlier, our goal is to estimate $w_J$ and $w_M$ from simulated trajectories of the MDP. Thus, it is constructive to consider projections onto $S_J$ and $S_M$ with respect to a norm that is weighted according to the state occupancy in these trajectories.

For a trajectory $x_0,\dots,x_{\tau-1}$, where $x_0$ is drawn from a fixed distribution $\zeta_0(x)$, and the states evolve according to the MDP with policy $\pi$, define the state occupancy probabilities
\begin{equation*}
    q_t(x) = P(x_t = x), \quad x \in X, \quad t = 0,1,\dots
\end{equation*}
and let
\begin{equation*}
\begin{split}
    q(x) &= \sum_{t=0}^{\infty} q_t(x), \quad x \in X \\
    Q &{\triangleq} diag(q).
\end{split}
\end{equation*}
We make the following assumption on the policy $\pi$ and initial distribution $\zeta_0$
\begin{assumption}\label{assumption:all_states_visited}
Each state has a positive probability of being visited, namely, $q(x)>0$ for all $x\in X$.
\end{assumption}

For vectors in $\R^n$, we introduce the weighted Euclidean norm
\begin{equation*}
    \|y\|_q = \sqrt{\sum_{i=1}^{n} q(i) \left( y(i) \right)^2}, \quad y \in \R^n,
\end{equation*}
and we denote by $\Pi_J$ and $\Pi_M$ the projections from $\R^n$ onto the subspaces $S_J$ and $S_M$, respectively, with respect to this norm. For $z\in \R^{2n}$ we denote by $\Pi$ the projection of $z_J$ onto $S_J$ and $z_M$ onto $S_M$, namely
\footnote{The projection operators $\Pi_J$ and $\Pi_M$ are linear, and may be written explicitly as $\Pi_J = \Phi_J(\Phi_J^T Q \Phi_J )^{-1} \Phi_J^T Q$, and similarly for $\Pi_M$.}
\begin{equation}\label{eq:Pi_JM}
    \Pi = \left(
                 \begin{array}{cc}
                   \Pi_J & 0 \\
                   0 & \Pi_M \\
                 \end{array}
               \right).
\end{equation}

We are now ready to fully describe our approximation scheme. We consider the \emph{projected} fixed point equation
\begin{equation}\label{eq:projected_fixed_point}
    z = \Pi T z,
\end{equation}
and, letting $z^*$ denote its solution, propose the approximate value function $\tilde{J} = z^*_J \in S_J$ and second moment function $\tilde{M} = z^*_M \in S_M$.

We proceed to derive some properties of the projected fixed point equation \eqref{eq:projected_fixed_point}. We begin by stating a well known result regarding the contraction properties of the \emph{projected Bellman operator} $\Pi_J T_J$, where $T_J y = r + Py$. A proof can be found at \cite{Ber2012DynamicProgramming}, proposition 7.1.1.
\begin{lemma}\label{lemma:PI_T_contraction}
Let Assumptions \ref{assumption:proper_policy}, \ref{assumption:Phi_rank}, and \ref{assumption:all_states_visited} hold. Then, there exists some norm $\| \cdot \|_J$ and some $\beta_J < 1$ such that
\begin{equation*}
  \|\Pi_J Py\|_J \leq \beta_J \|y\|_J, \quad \forall y \in \R^n.
\end{equation*}
Similarly, there exists some norm $\| \cdot \|_M$ and some $\beta_M < 1$ such that
\begin{equation*}
  \|\Pi_M Py\|_M \leq \beta_M \|y\|_M, \quad \forall y \in \mathbb{R}^n.
\end{equation*}
\end{lemma}

Next, we define a weighted norm on $\R^{2n}$
\begin{definition}
For a vector $z \in \R^{2n}$ and a scalar $0 < \alpha < 1$, the $\alpha$-weighted norm is
\begin{equation}\label{eq:alpha_norm}
  \|z\|_{\alpha} = \alpha \|z_J\|_J + (1-\alpha) \|z_M\|_M,
\end{equation}
where the norms $\| \cdot \|_J$ and $\| \cdot \|_M$ are defined in Lemma \ref{lemma:PI_T_contraction}.
\end{definition}
%Let us verify that this is indeed a norm. Let $X_1=(J_1,M_1)$, $X_2=(J_2,M_2)$, and $a\in\R$. We have
%\begin{equation*}
%\begin{split}
%    \|aX_1\|_{\alpha} &= \alpha \|aJ_1\|_J + (1-\alpha) \|aM_1\|_M \\
%                      &= a( \alpha \|J_1\|_J + (1-\alpha) \|M_1\|_M) \\
%                      &= a\|X_1\|_{\alpha}.
%\end{split}
%\end{equation*}
%\begin{equation*}
%\begin{split}
%    \|X_1+X_2\|_{\alpha} &= \alpha \|J_1+J_2\|_J + (1-\alpha) \|M_1+M_2\|_M \\
%                      &{\leq} \alpha \|J_1\|_J+\alpha \|J_2\|_J + (1-\alpha) \|M_1\|_M+ (1-\alpha) \|M_2\|_M\\
%                      &= \|X_1\|_{\alpha}+\|X_2\|_{\alpha}.
%\end{split}
%\end{equation*}
%\begin{equation*}
%\begin{split}
%    \|X_1\|_{\alpha} &= 0 \\
%    \alpha \|J_1\|_J + (1-\alpha) \|M_1\|_M &= 0 \\
%     J_1 &= 0 \\
%     M_1 &= 0 \\
%     X_1 &= 0.
%\end{split}
%\end{equation*}
Our main result of this section is given in the following lemma, where we show that the projected operator $\Pi T$ is a contraction with respect to the  $\alpha$-weighted norm.
\begin{lemma}\label{lemma:PI_T_JM_contraction}
Let Assumptions \ref{assumption:proper_policy}, \ref{assumption:Phi_rank}, and \ref{assumption:all_states_visited} hold. Then, there exists some $0<\alpha<1$ and some $\beta < 1$ such that $\Pi T$ is a $\beta$-contraction with respect to the $\alpha$-weighted norm, i.e.,
\begin{equation*}
  \|\Pi T z\|_{\alpha} \leq \beta \|z\|_{\alpha}, \quad \forall z \in \R^{2n}.
\end{equation*}
\end{lemma}

\begin{proof}
Let $\cal{P}$ denote the following matrix in $\R^{2n \times 2n}$
\begin{equation*}
    \cal{P} = \left( \begin{array}{cc}
                P & 0 \\
                2RP & P\\
                \end{array} \right),
\end{equation*}
and let $z\in\R^{2n}$. We need to show that
\begin{equation*}
  \|\Pi {\cal P} z\|_{\alpha} \leq \beta \|z\|_{\alpha}.
\end{equation*}
From \eqref{eq:Pi_JM} we have
\begin{equation*}
    \Pi {\cal P} = \left( \begin{array}{cc}
                    \Pi_J P & 0 \\
                    2\Pi_M R P & \Pi_M P\\
                \end{array} \right).
\end{equation*}
Therefore, we have
\begin{equation}\label{eq:cont_proof_1}
    \begin{split}
       \|\Pi {\cal P} z\|_{\alpha} =& \alpha \|\Pi_J P z_J\|_J \\ &+ (1-\alpha) \|2\Pi_M R P z_J + \Pi_M P z_M\|_M \\
         {\leq}& \alpha \|\Pi_J P z_J\|_J \\ &+ (1-\alpha) \|\Pi_M P z_M\|_M \\ &+ (1-\alpha) \|2\Pi_M R P z_J\|_M \\
         {\leq}& \alpha \beta_J \|z_J\|_J \\ &+ (1-\alpha) \beta_M \|z_M\|_M \\ &+ (1-\alpha) \|2\Pi_M R P z_J\|_M ,\\
     \end{split}
\end{equation}
where the equality is by definition of the $\alpha$ weighted norm \eqref{eq:alpha_norm}, the first inequality is from the triangle inequality, and the second inequality is by Lemma \ref{lemma:PI_T_contraction}. Now, we claim that there exists some finite $C$ such that
\begin{equation}\label{eq:cont_proof_2}
    \|2\Pi_M R P y\|_M \leq C \|y\|_J, \quad \forall y \in \R^n.
\end{equation}
To see this, note that since $\R^{n}$ is a finite dimensional real vector space, all vector norms are equivalent \cite{HornJohnson} therefore there exist finite $C_1$ and $C_2$ such that for all $y \in \R^n$
\begin{equation*}
    C_1 \|2\Pi_M R P y\|_2 \leq \|2\Pi_M R P y\|_M \leq C_2 \|2\Pi_M R P y\|_2,
\end{equation*}
where $\|\cdot\|_2$ denotes the Euclidean norm. Let $\lambda$ denote the spectral norm of the matrix $2\Pi_M R P$, which is finite since all the matrix elements are finite. We have
\begin{equation*}
    \|2\Pi_M R P y\|_2 \leq \lambda \|y\|_2, \quad \forall y \in \R^n.
\end{equation*}
Using again the fact that all vector norms are equivalent, there exists a finite $C_3$ such that
\begin{equation*}
    \|y\|_2 \leq C_3 \|y\|_J, \quad \forall y \in \R^n.
\end{equation*}
Setting $C = C_2 \lambda C_3$ we get the desired bound. Let $\tilde{\beta} = \max \{\beta_J, \beta_M \} < 1$, and choose $\epsilon > 0$ such that
\begin{equation*}
    \tilde{\beta}+\epsilon < 1.
\end{equation*}
Now, choose $\alpha$ such that
\begin{equation*}
    \alpha = \frac{C}{\epsilon+C}.
\end{equation*}
We have that
\begin{equation*}
    (1- \alpha) C = \alpha \epsilon,
\end{equation*}
and plugging in \eqref{eq:cont_proof_2}
\begin{equation*}
    (1-\alpha) \|2\Pi_M R P y\|_M \leq \alpha \epsilon \|y\|_J.
\end{equation*}
Plugging in \eqref{eq:cont_proof_1} we have
\begin{equation*}
    \begin{split}
        & \alpha \beta_J \|z_J\|_J + (1-\alpha) \beta_M \|z_M\|_M + (1-\alpha) \|2\Pi_M R P z_J\|_M \\
        &{\leq} \alpha \beta_J \|z_J\|_J + (1-\alpha) \beta_M \|z_M\|_M + \alpha \epsilon \|z_J\|_J \\
        &{\leq} (\tilde{\beta}+\epsilon)\left(\alpha \|z_J\|_J + (1-\alpha) \|z_M\|_M \right)\\
     \end{split}
\end{equation*}
and therefore
\begin{equation*}
       \|\Pi {\cal P} z\|_{\alpha} \leq (\tilde{\beta}+\epsilon) \| z\|_{\alpha}
\end{equation*}
Finally, choose $\beta = \tilde{\beta} + \epsilon$.
\end{proof}

Lemma \ref{lemma:PI_T_JM_contraction} guarantees that the projected operator $\Pi T$ has a unique fixed point. Let us denote this fixed point by $z^*$, and let $w^*_J,w^*_M$ denote the corresponding weights, which are unique due to Assumption \ref{assumption:Phi_rank}
\begin{equation}\label{eq:PI_T_JM_Fixed_point}
    \begin{split}
      \Pi T z^* &= z^*, \\
      z^*_J &= \Phi_J w^*_J, \\
      z^*_M &= \Phi_M w^*_M.
    \end{split}
\end{equation}

In the next lemma we provide a bound on the approximation error. The proof is in Appendix \ref{supp:lemma:PI_T_JM_error}.
\begin{lemma}\label{lemma:PI_T_JM_error}
Let Assumptions \ref{assumption:proper_policy}, \ref{assumption:Phi_rank}, and \ref{assumption:all_states_visited} hold. Denote by $z_{true}\in \R^{2n}$ the true value and second moment functions, i.e., $z_{true}$ satisfies $z_{true} = T z_{true}$. Then,
\begin{equation*}
  \|z_{true} - z^*\|_{\alpha} \leq \frac{1}{1-\beta} \|z_{true} - \Pi z_{true}\|_{\alpha},
\end{equation*}
with $\alpha$ and $\beta$ defined in Lemma \ref{lemma:PI_T_JM_contraction}.
\end{lemma}

%%%%%%%%%%%%%%%%%%%%%%%%%%%%%%%%%%%
\section{Simulation Based Estimation Algorithms}\label{sec:var_estimation}
%%%%%%%%%%%%%%%%%%%%%%%%%%%%%%%%%%%
We now use the theoretical results of the previous subsection to derive simulation based algorithms for jointly estimating the value function and second moment. The projected equation \eqref{eq:PI_T_JM_Fixed_point} is linear, and can be written in matrix form as follows. First let us write the equation explicitly as
\begin{equation}\label{eq:projected_eq_explicit}
    \begin{split}
      \Pi_{J} \left( r + P \Phi_J w^*_J \right) &= \Phi_J w^*_J, \\
      \Pi_{M} \left( Rr + 2R P \Phi_J w^*_J + P \Phi_M w^*_M \right) &= \Phi_M w^*_M. \\
    \end{split}
\end{equation}

Projecting a vector $y$ onto $ \Phi w $ satisfies the following orthogonality condition
\begin{equation*}
        \Phi^T Q (y - \Phi w) = 0,
\end{equation*}
therefore we have
\begin{equation*}
    \begin{split}
      \Phi_J^T Q \left( \Phi_J w^*_J - \left( r + P \Phi_J w^*_J \right) \right) &= 0, \\
      \Phi_M^T Q \left( \Phi_M w^*_M - \left( Rr + 2R P \Phi_J w^*_J + P \Phi_M w^*_M \right) \right) &= 0, \\
    \end{split}
\end{equation*}
which can be written as
\begin{equation}\label{eq:fixed_point_matrix_form}
    \begin{split}
      A w^*_J &= b, \\
      C w^*_M &= d, \\
    \end{split}
\end{equation}
with
\begin{equation}\label{eq:fixed_point_matrix_form2}
    \begin{split}
    A &= \Phi_J^T Q \left( I - P \right) \Phi_J, \quad b = \Phi_J^T Q r,\\
    C &= \Phi_M^T Q \left( I - P \right) \Phi_M, \quad d = \Phi_M^T Q R \left( r + 2P \Phi_J A^{-1}b \right),\\
    \end{split}
\end{equation}
and the matrices $A$ and $C$ are invertible since Lemma \ref{lemma:PI_T_JM_contraction} guarantees a unique solution to \eqref{eq:PI_T_JM_Fixed_point} and Assumption \ref{assumption:Phi_rank} guarantees the unique weights of its projection.

\subsection{A Least Squares TD Algorithm}

Our first simulation based algorithm is an extension of the Least Squares Temporal Difference (LSTD) algorithm \cite{boyan2002technical}. We simulate $N$ trajectories of the MDP with the policy $\pi$ and initial state distribution $\zeta_0$. Let $x^k_0,x^k_1,\dots,x^k_{\tau^k-1}$ and $\tau^k$, where $k=0,1,\ldots,N$, denote the state sequence and visit times to the terminal state within these trajectories, respectively. We now use these trajectories to form the following estimates of the terms in \eqref{eq:fixed_point_matrix_form2}
\begin{equation}\label{eq:sampled_AbCd}
    \begin{split}
    A_N &= \E_N \left[\sum_{t=0}^{\tau-1} \phi_J(x_t) (\phi_J(x_{t}) - \phi_J(x_{t+1}))^T \right], \\
    b_N &= \E_N \left[\sum_{t=0}^{\tau-1} \phi_J(x_t) r(x_t)\right], \\
    C_N &= \E_N \left[\sum_{t=0}^{\tau-1} \phi_M(x_t) (\phi_M(x_{t}) - \phi_M(x_{t+1}))^T\right], \\
    d_N &= \E_N \left[\sum_{t=0}^{\tau-1} \phi_M(x_t) r(x_t)\left( r(x_t) + 2 \phi_J(x_{t+1})^T A_N^{-1}b_N \right)\right],
    \end{split}
\end{equation}
where $\E_N$ denotes an empirical average over trajectories, i.e., $\E_N \left[f(x,\tau)\right] = \frac{1}{N}\sum_{k=1}^N f(x^k,\tau^k)$.
The LSTD approximation is given by
\begin{equation*}
    \begin{split}
      \hat{w}^*_{J} &= A_N^{-1}b_N, \\
      \hat{w}^*_{M} &= C_N^{-1}d_N. \\
    \end{split}
\end{equation*}
The next theorem shows that the LSTD approximation converges.
\begin{theorem}\label{th:LSTD_converges}
Let Assumptions \ref{assumption:proper_policy}, \ref{assumption:Phi_rank}, and \ref{assumption:all_states_visited} hold. Then $\hat{w}^*_{J} \to {w}^*_{J}$ and $\hat{w}^*_{M} \to {w}^*_{M}$ as $N \to \infty$ with probability 1.
\end{theorem}
The proof involves a straightforward application of the law of large numbers and is described in Appendix \ref{supp:th:LSTD_converges}.

\subsection{An online TD(0) Algorithm}

Our second estimation algorithm is an extension of the well known TD(0) algorithm \cite{SutBar98}. Again, we simulate trajectories of the MDP corresponding to the policy $\pi$ and initial state distribution $\zeta_0$, and we iteratively update our estimates at every visit to the terminal state\footnote{An extension to an algorithm that updates at every state transition is also possible, but we do not pursue such here.}. For some $0\leq t < \tau^{k}$ and weights $w_J,w_M$, we introduce the TD terms
\begin{equation*}
    \begin{split}
        \delta_{J}^k(t,w_J,w_M) =& r(x^k_t) + \left(\phi_J(x^k_{t+1})^T - \phi_J(x^k_t)^T\right) w_J, \\
        \delta_{M}^k(t,w_J,w_M) =& r^2(x^k_t) + 2r(x^k_t)\phi_J(x^k_{t+1})^T w_J \\
        &+ \left(\phi_M(x^k_{t+1})^T - \phi_M(x^k_t)^T\right) w_M.
    \end{split}
\end{equation*}
Note that $\delta_{J}^k$ is the standard TD error \cite{SutBar98}. The TD(0) update is given by
\begin{equation*}
    \begin{split}
      \hat{w}_{J;k+1} &= \hat{w}_{J;k} + \xi_k \sum_{t=0}^{\tau^{k}-1} \phi_J(x_t)\delta_{J}^k(t,\hat{w}_{J;k},\hat{w}_{M;k}), \\
      \hat{w}_{M;k+1} &= \hat{w}_{M;k} + \xi_k \sum_{t=0}^{\tau^{k}-1} \phi_M(x_t)\delta_{M}^k(t,\hat{w}_{J;k},\hat{w}_{M;k}), \\
    \end{split}
\end{equation*}
where $\{\xi_k\}$ are positive step sizes.

The next theorem shows that the TD(0) algorithm converges.
\begin{theorem}\label{th:TD_converges}
Let Assumptions \ref{assumption:proper_policy}, \ref{assumption:Phi_rank}, and \ref{assumption:all_states_visited} hold, and let the step sizes satisfy
\begin{equation*}
    \sum_{k=0}^{\infty}\xi_k = \infty, \quad \sum_{k=0}^{\infty}\xi_k^2 < \infty.
\end{equation*}
Then $\hat{w}_{J;k} \to {w}^*_{J}$ and $\hat{w}_{M;k} \to {w}^*_{M}$ as $k \to \infty$ with probability 1.
\end{theorem}
The proof, provided in Appendix \ref{supp:th:TD_converges}, is based on representing the TD(0) algorithm as a stochastic approximation and using contraction properties similar to the ones of the previous section to prove convergence.

%%%%%%%%%%%%%%%%%%%%%%%%%%%%%%%%%%%
\subsection{Multistep Algorithms}
%%%%%%%%%%%%%%%%%%%%%%%%%%%%%%%%%%%
A common method in value function approximation is to replace the single step mapping $T_J$ with a multistep version of the form
\begin{equation*}
T^{(\lambda)}_J = (1-\lambda)\sum_{l=0}^{\infty}\lambda^{l}T_J^{l+1}
\end{equation*}
with $0<\lambda<1$. The projected equation \eqref{eq:projected_eq_explicit} then becomes
\begin{equation*}
      \Pi_{J} T^{(\lambda)}_J \left( \Phi_J w^{*(\lambda)}_J \right) = \Phi_J w^{*(\lambda)}_J.
\end{equation*}
Similarly, we may write a multistep equation for $M$
\begin{equation}\label{eq:T_M_Lambda}
      \Pi_{M} T^{(\lambda)}_M \left( \Phi_M w^{*(\lambda)}_M \right) = \Phi_M w^{*(\lambda)}_M,
\end{equation}
where
\begin{equation*}
T^{(\lambda)}_M = (1-\lambda)\sum_{l=0}^{\infty}\lambda^{l}T_{M^*}^{l+1},
\end{equation*}
and
\begin{equation*}
  T_{M^*} \left( y \right) = Rr + 2RP \Phi_J w^{*(\lambda)}_J + P y.
\end{equation*}
Note the difference between $T_{M^*}$ and $T_M$ defined earlier; We are no longer working on the joint space of $J$ and $M$ but instead we have an independent equation for approximating $J$, and its solution $w^{*(\lambda)}_J$ is part of equation \eqref{eq:T_M_Lambda} for approximating $M$.
By Proposition 7.1.1. of \cite{Ber2012DynamicProgramming} both $\Pi_{J} T^{(\lambda)}_J$ and $\Pi_{M} T^{(\lambda)}_M$ are contractions with respect to the weighted norm $\|\cdot\|_q$, therefore both multistep projected equations admit a unique solution. In a similar manner to the single step version, the projected equations may be written in matrix form
\begin{equation}\label{eq:TD_Lambda_solution}
    \begin{split}
      A^{(\lambda)} w^{*(\lambda)}_J &= b^{(\lambda)}, \\
      C^{(\lambda)} w^{*(\lambda)}_M &= d^{(\lambda)}, \\
    \end{split}
\end{equation}
where
\begin{equation*}
    \begin{split}
    A^{(\lambda)} &= \Phi_J^T Q \left( I - P^{(\lambda)} \right) \Phi_J, \quad b^{(\lambda)} = \Phi_J^T Q (I-\lambda P)^{-1}r, \\
    C^{(\lambda)} &= \Phi_M^T Q \left( I - P^{(\lambda)} \right) \Phi_M, \\
    d^{(\lambda)} &= \Phi_M^T Q (I-\lambda P)^{-1} R \left( r + 2P \Phi_J w^{*(\lambda)}_J \right),
    \end{split}
\end{equation*}
and
\begin{equation*}
    P^{(\lambda)} = (1-\lambda)\sum_{l=0}^{\infty}\lambda^{l}P^{l+1}.\\
\end{equation*}
Simulation based estimates $A^{(\lambda)}_N$ and $b^{(\lambda)}_N$ of the expressions above may be obtained by the use of eligibility traces, as described in \cite{Ber2012DynamicProgramming}, and the LSTD($\lambda$) approximation is then given by $\hat{w}^{*(\lambda)}_{J} = (A^{(\lambda)}_N)^{-1}b^{(\lambda)}_N$. By substituting $w^{*(\lambda)}_{J}$ with $\hat{w}^{*(\lambda)}_{J}$ in the expression for $d^{(\lambda)}$, a similar procedure may be used to derive estimates $C^{(\lambda)}_N$ and $d^{(\lambda)}_N$, and to obtain the LSTD($\lambda$) approximation $\hat{w}^{*(\lambda)}_{M} = (C^{(\lambda)}_N)^{-1}d^{(\lambda)}_N$. Due to the similarity to the LSTD procedure in \eqref{eq:sampled_AbCd}, the exact details are omitted.

%%%%%%%%%%%%%%%%%%%%%%%%%%%%%%%%%%%
\section{Positive Variance as a Constraint in LSTD}\label{sec:positive_var}
%%%%%%%%%%%%%%%%%%%%%%%%%%%%%%%%%%%
The TD algorithms of the preceding section approximated $J$ and $M$ by the solution to the fixed point equation \eqref{eq:PI_T_JM_Fixed_point}. While Lemma \ref{lemma:PI_T_JM_error} provides us a bound on the approximation error of $\tilde{J}$ and $\tilde{M}$ measured in the $\alpha$-weighted norm, it does not guarantee that the approximated variance $\tilde{V}$, given by $\tilde{M}-\tilde{J}^2$, is positive for all states. If we are estimating $M$ as a means to infer $V$, it may be useful to include our prior knowledge that $V\geq0$ in the estimation process. In this section we propose to enforce this knowledge as a constraint in the projected fixed point equation.

The multistep equation for the second moment weights \eqref{eq:T_M_Lambda} may be written with the projection operator as an explicit minimization
\begin{equation*}
      w^{*(\lambda)}_M =
      \arg \min_w \|\Phi_M w - \left( \tilde{r} + \tilde{\Phi} w^{*(\lambda)}_M \right) \|_q ,
\end{equation*}
with
\begin{equation*}
\tilde{\Phi} = P^{(\lambda)} \Phi_M,
\end{equation*}
and
\begin{equation*}
\tilde{r} = (I-\lambda P)^{-1}\left(Rr + 2R P \Phi_J w^{*(\lambda)}_J \right).
\end{equation*}
%\begin{equation*}
%\begin{split}
%      w^*_M =& \arg \min_w \|\Phi_M w - \\
%      &\left( (I-\lambda P)^{-1}\left(Rr + 2R P \Phi_J w^{*(\lambda)}_J \right)
%      + P^{(\lambda)} \Phi_M w^*_M \right) \|_q
%\end{split}
%\end{equation*}
Requiring non negative variance in some state $x$ may be written as a linear constraint in $w^{*(\lambda)}_M$
\begin{equation*}
      \phi_M(x)^T w^{*(\lambda)}_M - (\phi_J(x)^T w^{*(\lambda)}_J)^2 \geq 0.
\end{equation*}
Let $\{ x_1,\dots,x_l\}$ denote a set of states in which we demand that the variance be non negative. Let $H\in\R^{l \times l_M}$ denote a matrix with the features $-\phi_M^T(x_i)$ as its rows, and let $g\in\R^l$ denote a vector with elements $-(\phi_J(x_i)^T w^{*(\lambda)}_J)^2$.
%Also, for notational brevity let $\Phi' = P^{(\lambda)} \Phi_M$ and $\tilde{r} = (I-\lambda P)^{-1}\left(Rr + 2R P \Phi_J w^{*(\lambda)}_J \right)$.
We can write the variance-constrained projected equation for the second moment as
\begin{equation}\label{eq:constrained_projected_eq}
      w^{vc}_M = \begin{cases}
      \arg \min_w& \|\Phi_M w - \left( \tilde{r} + \tilde{\Phi} w^{vc}_M \right) \|_q \\
      \textrm{s.t.}& Hw \leq g
      \end{cases}
\end{equation}
The following assumption guarantees that the constraints in \eqref{eq:constrained_projected_eq} admit a feasible solution.
\begin{assumption}\label{assumption:feasible_solution}
There exists $w$ such that $Hw < g$.
\end{assumption}
Note that a simple way to satisfy Assumption \ref{assumption:feasible_solution} is to have some feature vector that is positive for all states.
Equation \eqref{eq:constrained_projected_eq} is a form of projected equation studied in \cite{bertsekas2011temporal}, the solution of which may be obtained by the following iterative procedure
\begin{equation}\label{eq:VI_algorithm}
    w_{k+1} = \Pi_{\Xi,\hat{W}_M} [w_k - \gamma \Xi^{-1} (C^{(\lambda)} w_k - d^{(\lambda)}) ],
\end{equation}
where $\Xi$ is some positive definite matrix, and $\Pi_{\Xi,\hat{W}_M}$ denotes a projection onto the convex set $\hat{W}_M = \{ w | Hw \leq g \}$ with respect to the $\Xi$ weighted Euclidean norm.
The following lemma, which is based on a convergence result of \cite{bertsekas2011temporal}, guarantees that algorithm \eqref{eq:VI_algorithm} converges.
\begin{lemma}
Assume $\lambda>0$. Then there exists $\bar{\gamma}>0$ such that $\forall \gamma \in (0,\bar{\gamma})$ the algorithm \eqref{eq:VI_algorithm} converges at a linear rate to $w^{vc}_M$.
\end{lemma}
\begin{proof}
This is a direct application of the convergence result in \cite{bertsekas2011temporal}. The only nontrivial assumption that needs to be verified is that $T^{(\lambda)}_M$ is a contraction in the $\|\cdot\|_q$ norm (Proposition 1 in \citealp{bertsekas2011temporal}). For $\lambda > 0$ Proposition 7.1.1. of \cite{Ber2012DynamicProgramming} guarantees  that $T^{(\lambda)}_M$ is indeed contracting in the $\|\cdot\|_q$ norm.
\end{proof}

We illustrate the effect of the positive variance constraint in a simple example. Consider the Markov chain depicted in Figure \ref{fig:chain_MDP}, which consists of $N$ states with reward $-1$ and a terminal state $x^*$ with zero reward. The transitions from each state is either to a subsequent state (with probability $p$) or to a preceding state (with probability $1-p$), with the exception of the first state which transitions to itself instead. We chose to approximate $J$ and $M$ with polynomials of degree 1 and 2, respectively. For such a small problem the fixed point equation \eqref{eq:TD_Lambda_solution} may be solved exactly, yielding the approximation depicted in Figure \ref{fig:chain_JMV} (dotted line), for $p=0.7$, $N=30$, and $\lambda=0.95$. Note that the variance is negative for the last two states. Using algorithm \eqref{eq:VI_algorithm} we obtained a positive variance constrained approximation, which is depicted in figure \ref{fig:chain_JMV} (dashed line). Note that the variance is now positive for all states (as was required by the constraints).
\begin{figure}[h]
\begin{center}
\includegraphics[scale=0.3]{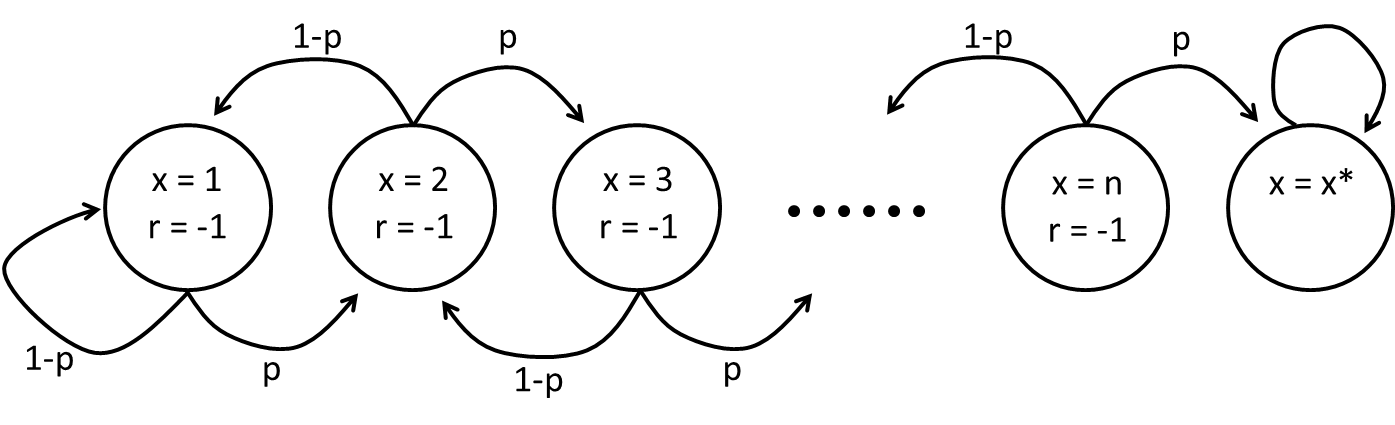}
\end{center}
\caption{A Markov chain\label{fig:chain_MDP}   }
\end{figure}

\begin{figure*}[t]
\begin{center}
\includegraphics[scale=0.35, trim=110 250 110 250, clip=true]{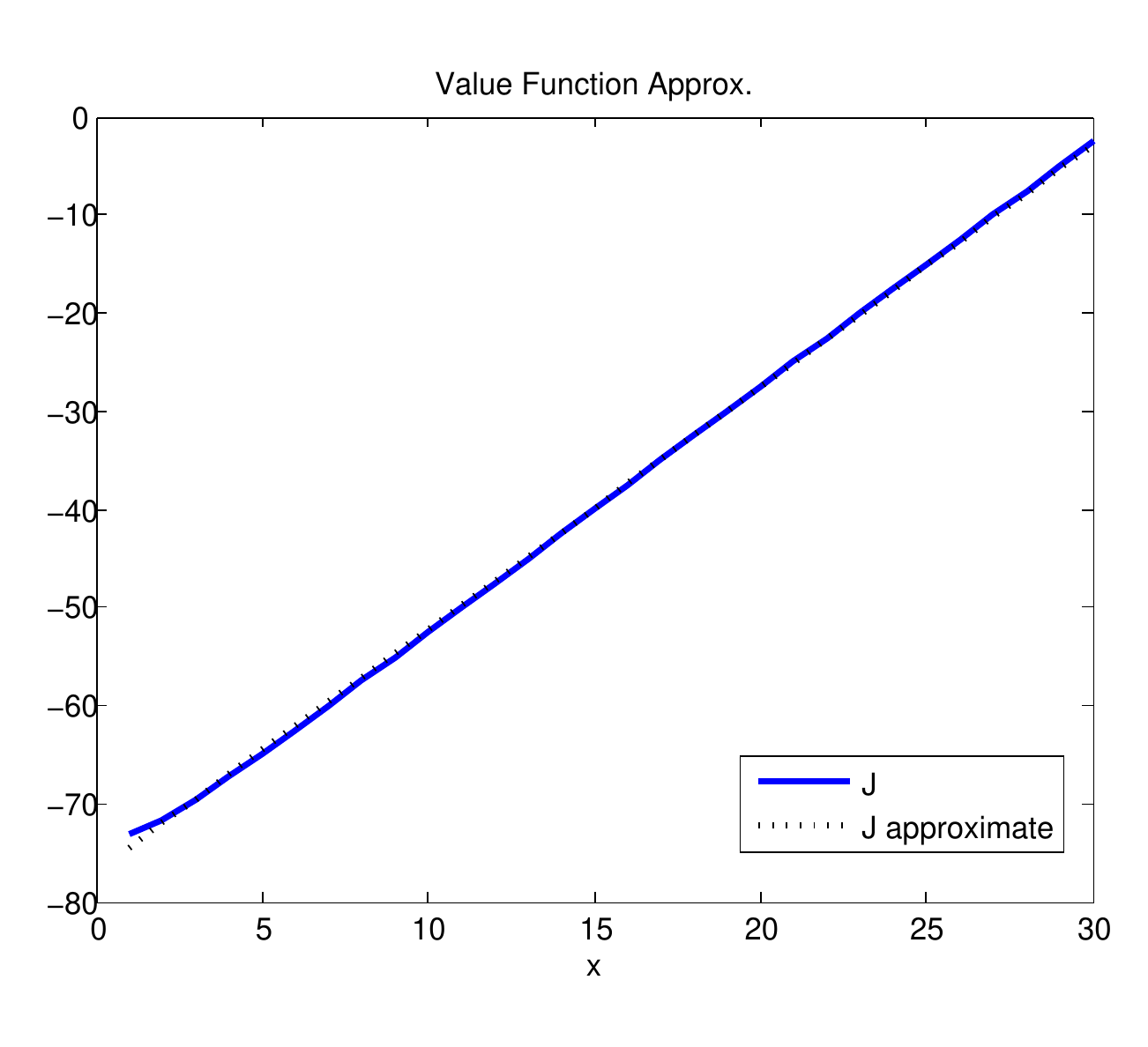}
\includegraphics[scale=0.35, trim=110 250 120 250, clip=true]{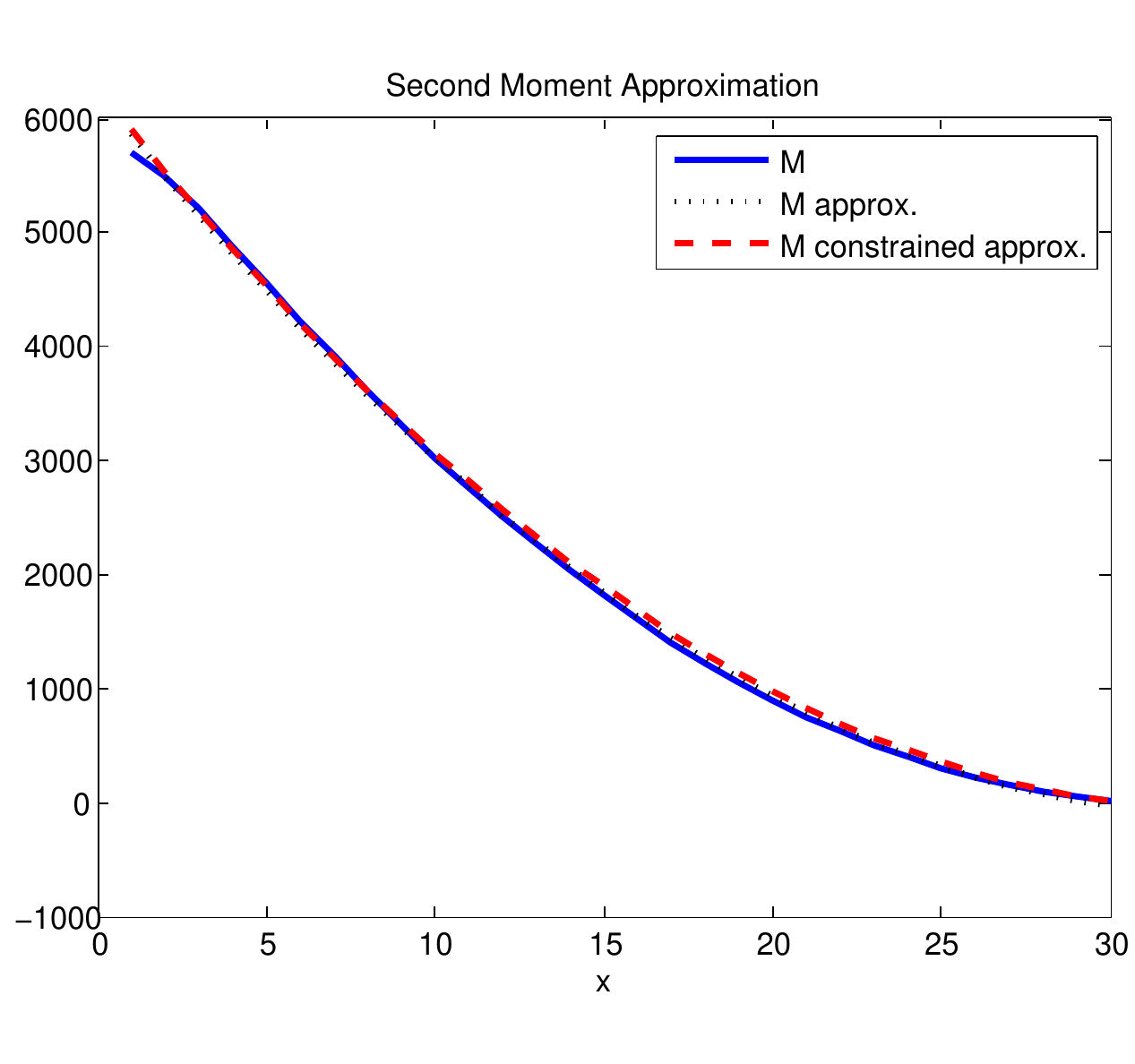}
\includegraphics[scale=0.35, trim=120 250 120 250, clip=true]{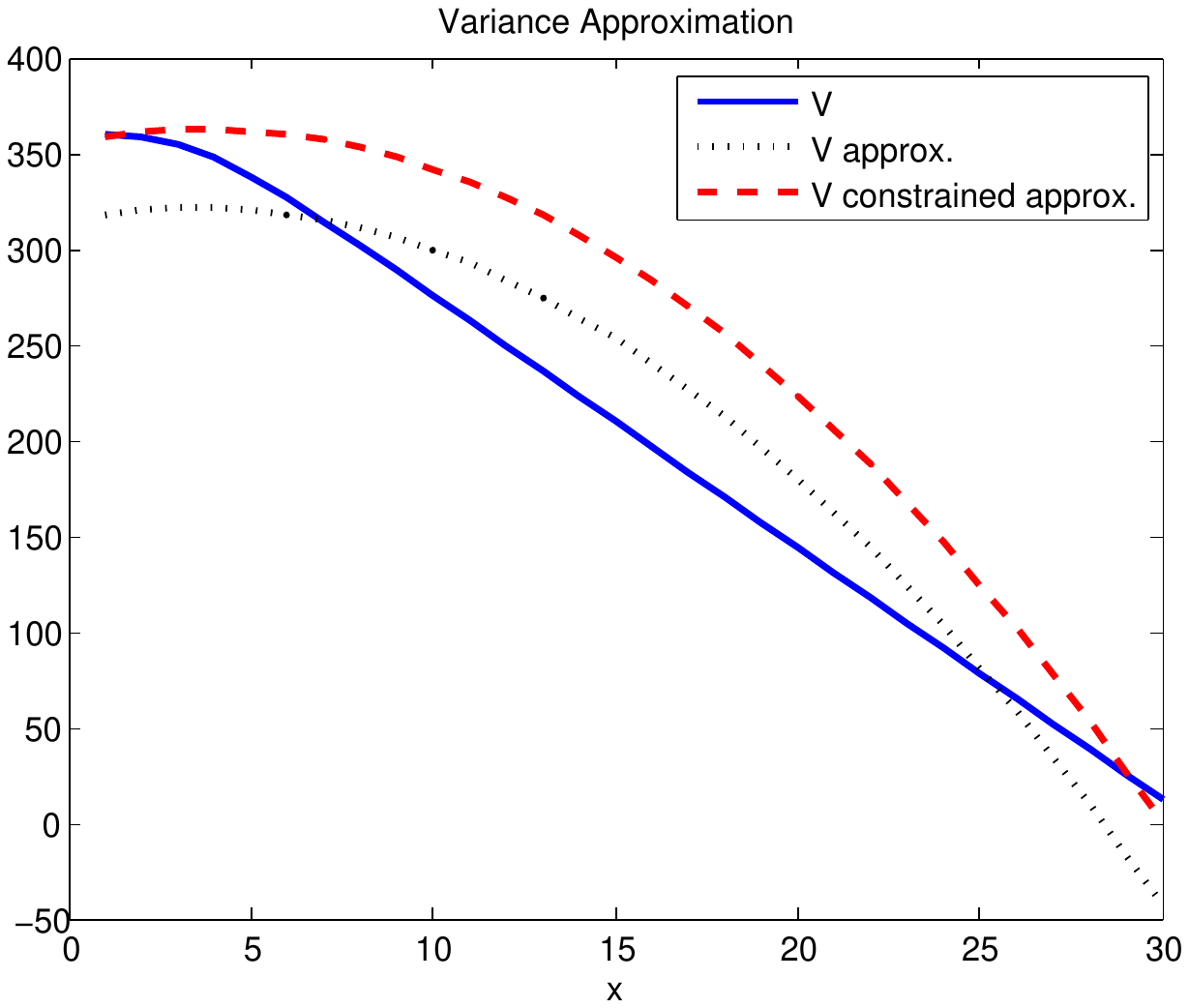}
\caption{\label{fig:chain_JMV}Value, second moment and variance approximation}
\end{center}
\end{figure*}

%%%%%%%%%%%%%%%%%%%%%%%%%%%%%%%%%%%
\section{Experiments}\label{sec:experiments}
%%%%%%%%%%%%%%%%%%%%%%%%%%%%%%%%%%%
In this section we present numerical simulations of policy evaluation on a challenging continuous maze domain. The goal of this presentation is twofold; first, we show that the variance function may be estimated successfully on a large domain using a reasonable amount of samples. Second, the intuitive maze domain highlights the information that may be gleaned from the variance function.
We begin by describing the domain and then present our policy evaluation results.

The Pinball Domain \citep{KonidarisSkillChaining09} is a continuous 2-dimensional maze where a small ball needs to be maneuvered between obstacles to reach some target area, as depicted in figure \ref{fig:pinball_domain} (left). The ball is controlled by applying a constant force in one of the 4 directions at each time step, which causes acceleration in the respective direction. In addition, the ball's velocity is susceptible to additive Gaussian noise (zero mean, standard deviation 0.03) and friction (drag coefficient 0.995). The state of the ball is thus 4-dimensional ($x,y,\dot{x},\dot{y}$), and the action set is discrete, with 4 available controls. The obstacles are sharply shaped and fully elastic, and collisions cause the ball to bounce. As noted in \cite{KonidarisSkillChaining09}, the sharp obstacles and continuous dynamics make the pinball domain more challenging for RL than simple navigation tasks or typical benchmarks like Acrobot.

A Java implementation of the pinball domain used in \citep{KonidarisSkillChaining09} is available on-line \footnote{http://people.csail.mit.edu/gdk/software.html} and was used for our simulations as well, with the addition of noise to the velocity.

We obtained a near-optimal policy using SARSA \citep{SutBar98} with radial basis function features and a reward of -1 for all states until reaching the target. The value function for this policy is plotted in Figure \ref{fig:pinball_domain}, for states with zero velocity. As should be expected, the value is approximately a linear function of the distance to the target.

Using 3000 trajectories (starting from uniformly distributed random states in the maze) we estimated the value and second moment functions by the LSTD($\lambda$) algorithm described above. We used uniform tile coding as features ($50\times 50$ non-overlapping tiles in $x$ and $y$, no dependence on velocity) and set $\lambda=0.9$. The resulting estimated standard deviation function is shown in Figure \ref{fig:pinball_variance} (left). In comparison, the standard deviation function shown in Figure \ref{fig:pinball_variance} (right) was estimated by the naive sample variance, and required 500 trajectories from each point - a total of 1,250,000 trajectories.

Note that the variance function is clearly not a linear function of the distance to the target, and in some places not even monotone. Furthermore, we see that an area in the top part of the maze before the first turn is very risky, even more than the farthest point from the target. We stress that this information cannot be gleaned from inspecting the value function alone.

\begin{figure*}
\begin{center}
\includegraphics[scale=0.29 , trim=0 -40 0 0, clip=true]{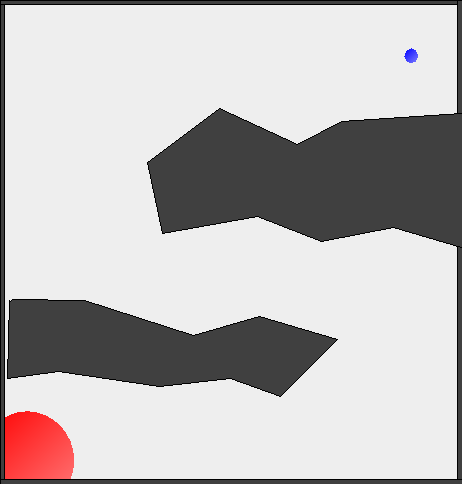}
\includegraphics[scale=0.4]{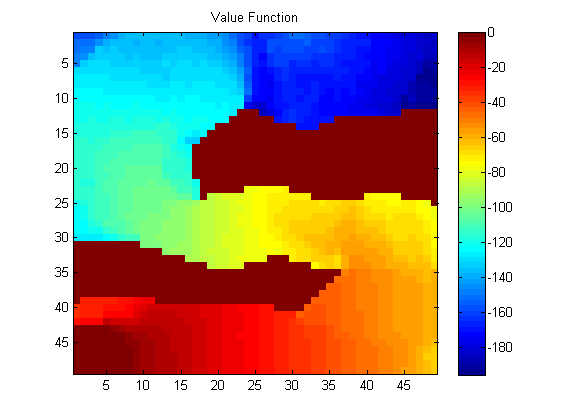}
\caption{\label{fig:pinball_domain}The pinball domain}
\end{center}
\end{figure*}

\begin{figure*}
\begin{center}
\includegraphics[scale=0.4]{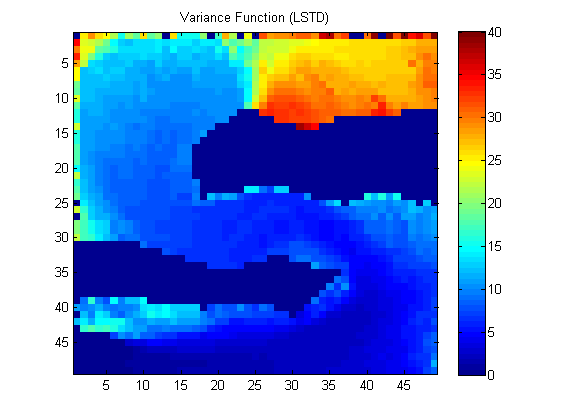}
\includegraphics[scale=0.4]{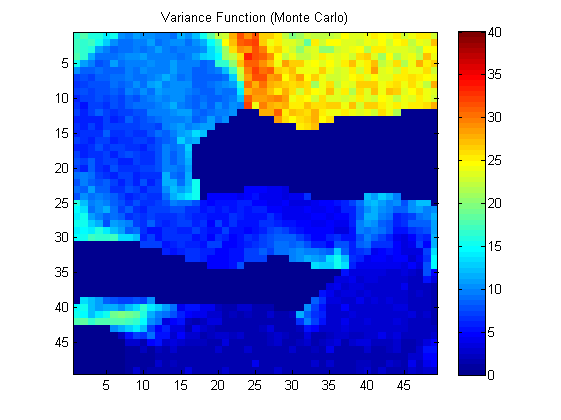}
\caption{\label{fig:pinball_variance}Standard Deviation of Reward To Go}
\end{center}
\end{figure*}

%%%%%%%%%%%%%%%%%%%%%%%%%%%%%%%%%%%%%%%%%%%%%%%%%%%%%%%%%%%%%%%%%%%%%%%%
\section{Conclusion}\label{sec:conclusion}
%%%%%%%%%%%%%%%%%%%%%%%%%%%%%%%%%%%%%%%%%%%%%%%%%%%%%%%%%%%%%%%%%%%%%%%%
This work presented a novel framework for policy evaluation in RL with variance related performance criteria. We presented both formal guarantees and empirical evidence that this approach is useful in problems with a large state space.

A few issues are in need of further investigation. First, we note a possible extension to other risk measures such as the percentile criterion \cite{journals/ior/DelageM10}. In a recent work, \citet{morimura2012parametric} derived Bellman equations for the \emph{distribution} of the total return, and appropriate TD learning rules were proposed, albeit without function approximation and formal guarantees.

More importantly, at the moment it remains unclear how the variance function may be used for \emph{policy optimization}. While a naive policy improvement step may be performed, its usefulness should be questioned, as it was shown to be problematic for the standard deviation adjusted reward \cite{sobel1982variance} and the variance constrained reward \cite{mannor2011mean}. In \cite{Tamar2012mean_var}, a policy gradient approach was proposed for handling variance related criteria, which may be extended to an actor-critic method by using the variance function presented here.

\bibliography{VarTD_ICML2013}
\bibliographystyle{icml2013}

\newpage
\appendix
\onecolumn
\textbf{Supplementary Material}
%%%%%%%%%%%%%%%%%%%%%%%%%%%%%%%%%%%%%%%%%%%%%%%%%%%%%%%%%%%%%%%%%%%%%%%%
\section{Proof of Proposition \ref{prop:equations}\label{supp:prop:equations}}
%%%%%%%%%%%%%%%%%%%%%%%%%%%%%%%%%%%%%%%%%%%%%%%%%%%%%%%%%%%%%%%%%%%%%%%%
\begin{proof}
The equation for $J(x)$ is well-known, and its proof is given here only for completeness.
Choose $x\in X$. Then,
\begin{equation*}
\begin{split}
  J(x) &= \E\left[B | x_{0}=x\right] \\
    &= \E\left[\sum_{k=0}^{\tau-1}r(x_k) | x_{0}=x\right] \\
    &= r(x) + \E\left[\sum_{k=1}^{\tau-1}r(x_k) | x_{0}=x\right] \\
    &= r(x) + \E\left[\E\left[\sum_{k=1}^{\tau-1}r(x_k) | x_{0}=x,x_{1}=y\right]\right] \\
    &= r(x) + \sum_{y\in X} P(y|x) J(y)
\end{split}
\end{equation*}
where we excluded the terminal state from the sum since reaching it ends the trajectory.

Similarly,
\begin{equation*}
\begin{split}
  M(x) &= \E\left[B^2 | x_{0}=x\right] \\
    &= \E\left[\left( \sum_{k=0}^{\tau-1}r(x_k) \right)^2 | x_{0}=x\right] \\
    &= \E\left[\left( r(x_0) + \sum_{k=1}^{\tau-1}r(x_k) \right)^2 | x_{0}=x\right] \\
    &= r(x)^2 + 2r(x)\E\left[\sum_{k=1}^{\tau-1}r(x_k) | x_{0}=x\right] + \E\left[\left( \sum_{k=1}^{\tau-1}r(x_k) \right)^2 | x_{0}=x\right] \\
    &= r(x)^2 + 2r(x) \sum_{y\in X}P(y|x) J(y)+ \sum_{y\in X}P(y|x) M(y).
\end{split}
\end{equation*}

The uniqueness of the value function $J$ for a proper policy is well known, c.f. proposition 3.2.1 in \cite{Ber2012DynamicProgramming}. The uniqueness of $M$ follows by observing that in the equation for $M$, $M$ may be seen as the value function of an MDP with the same transitions but with reward $r(x)^2 + 2r(x) \sum_{y\in X}P(y|x) J(y)$. Since only the rewards change, the policy remains proper and proposition 3.2.1 in \cite{Ber2012DynamicProgramming} applies.
\end{proof}
%%%%%%%%%%%%%%%%%%%%%%%%%%%%%%%%%%%%%%%%%%%%%%%%%%%%%%%%%%%%%%%%%%%%%%%%
\section{Proof of Lemma \ref{lemma:PI_T_JM_error}\label{supp:lemma:PI_T_JM_error}}
%%%%%%%%%%%%%%%%%%%%%%%%%%%%%%%%%%%%%%%%%%%%%%%%%%%%%%%%%%%%%%%%%%%%%%%%
\begin{proof}
We have
\begin{equation*}
\begin{split}
  \|z_{true} - z^*\|_{\alpha} &{\leq} \|z_{true} - \Pi z_{true}\|_{\alpha} + \|\Pi z_{true} - z^*\|_{\alpha} \\
  &= \|z_{true} - \Pi z_{true}\|_{\alpha} + \|\Pi T z_{true} - \Pi T z^*\|_{\alpha} \\
  &{\leq} \|z_{true} - \Pi z_{true}\|_{\alpha} + \beta\|z_{true} - z^*\|_{\alpha}.
\end{split}
\end{equation*}
rearranging gives the stated result.
\end{proof}

%%%%%%%%%%%%%%%%%%%%%%%%%%%%%%%%%%%%%%%%%%%%%%%%%%%%%%%%%%%%%%%%%%%%%%%%
\section{Proof of Theorem \ref{th:LSTD_converges}\label{supp:th:LSTD_converges}}
%%%%%%%%%%%%%%%%%%%%%%%%%%%%%%%%%%%%%%%%%%%%%%%%%%%%%%%%%%%%%%%%%%%%%%%%
\begin{proof}
Let $\phi_1(x)$, $\phi_2(x)$ be some vector functions of the state. We claim that
\begin{equation}\label{eq:LSTD_proof_1}
    \mathbb{E} \left[ \sum_{t=0}^{\tau-1} \phi_1(x_t) \phi_2(x_t)^T \right] = \sum_{x} q(x) \phi_1(x) \phi_2(x)^T.
\end{equation}
To see this, let $\mathbbm{1}(\cdot)$ denote the indicator function and write
\begin{equation*}
    \begin{split}
    \mathbb{E} \left[ \sum_{t=0}^{\tau-1} \phi_1(x_t) \phi_2(x_t)^T \right] &= \mathbb{E} \left[ \sum_{t=0}^{\tau-1} \sum_x \phi_1(x) \phi_2(x)^T \mathbbm{1}(x_t = x) \right] \\
    &= \mathbb{E} \left[ \sum_x \phi_1(x) \phi_2(x)^T \sum_{t=0}^{\tau-1} \mathbbm{1}(x_t = x) \right] \\
    &= \sum_x \phi_1(x) \phi_2(x)^T \mathbb{E} \left[ \sum_{t=0}^{\tau-1} \mathbbm{1}(x_t = x) \right].
    \end{split}
\end{equation*}
Now, note that the last term on the right hand side is an expectation (over all possible trajectories) of the number of visits to a state $x$ until reaching the terminal state, which is exactly $q(x)$ since
\begin{equation*}
    \begin{split}
    q(x) &= \sum_{t=0}^{\infty} P(x_t = x) \\
         &= \sum_{t=0}^{\infty} \mathbb{E}[\mathbbm{1}(x_t = x)] \\
         &= \mathbb{E} \left[ \sum_{t=0}^{\infty} \mathbbm{1}(x_t = x) \right] \\
         &= \mathbb{E} \left[ \sum_{t=0}^{\tau-1} \mathbbm{1}(x_t = x) \right],
    \end{split}
\end{equation*}
where the last equality follows from the absorbing property of the terminal state.
Similarly, we have
\begin{equation}\label{eq:LSTD_proof_2}
    \mathbb{E} \left[ \sum_{t=0}^{\tau-1} \phi_1(x_t) \phi_2(x_{t+1})^T \right] = \sum_{x}\sum_{y} q(x) P(y|x) \phi_1(x) \phi_2(y)^T,
\end{equation}
since
\begin{equation*}
    \begin{split}
    \mathbb{E} \left[ \sum_{t=0}^{\tau-1} \phi_1(x_t) \phi_2(x_{t+1})^T \right] &= \mathbb{E} \left[ \sum_{t=0}^{\tau-1} \sum_x \sum_y \phi_1(x) \phi_2(y)^T \mathbbm{1}(x_t = x,x_{t+1} = y) \right] \\
    &= \mathbb{E} \left[ \sum_x \sum_y \phi_1(x) \phi_2(y)^T \sum_{t=0}^{\tau-1} \mathbbm{1}(x_t = x,x_{t+1} = y) \right] \\
    &= \sum_x \sum_y \phi_1(x) \phi_2(y)^T \mathbb{E} \left[ \sum_{t=0}^{\tau-1} \mathbbm{1}(x_t = x,x_{t+1} = y) \right]
    \end{split}
\end{equation*}
and
\begin{equation*}
    \begin{split}
    q(x)P(y|x) &= \sum_{t=0}^{\infty} P(x_t = x)P(y|x) \\
         &= \sum_{t=0}^{\infty} P(x_t = x, x_{t+1} = y) \\
         &= \sum_{t=0}^{\infty} \mathbb{E}[\mathbbm{1}(x_t = x, x_{t+1} = y)] \\
         &= \mathbb{E} \left[ \sum_{t=0}^{\infty} \mathbbm{1}(x_t = x, x_{t+1} = y) \right] \\
         &= \mathbb{E} \left[ \sum_{t=0}^{\tau-1} \mathbbm{1}(x_t = x, x_{t+1} = y) \right]. \\
    \end{split}
\end{equation*}
Since trajectories between visits to the recurrent state are statistically independent, the law of large numbers together with the expressions in \eqref{eq:LSTD_proof_1} and \eqref{eq:LSTD_proof_2} suggest that the approximate expressions in \eqref{eq:sampled_AbCd} converge to their expected values with probability 1, therefore we have
\begin{equation*}
    \begin{split}
    A_N &{\to} A, \quad b_N {\to} b, \\
    C_N &{\to} C, \quad d_N {\to} D,
    \end{split}
\end{equation*}
and
\begin{equation*}
    \begin{split}
      \hat{w}^*_{J;N} &= A_N^{-1}b_N {\to} A^{-1}b = {w}^*_{J}, \\
      \hat{w}^*_{M;N} &= C_N^{-1}d_N {\to} C^{-1}d = {w}^*_{M}. \\
    \end{split}
\end{equation*}
\end{proof}

%\begin{equation}\label{eq:LSTD_approx_exp}
%    \begin{split}
%    \Phi_J^T Q \Phi_J &\approx \frac{1}{N+1}\sum_{k=0}^{N}\sum_{t=\tau^k}^{\tau^{k+1}-1} \phi_J(x_t) \phi_J(x_t)^T, \\
%    \Phi_J^T QP' \Phi_J &\approx \frac{1}{N+1}\sum_{k=0}^{N}\sum_{t=\tau^k}^{\tau^{k+1}-1} \phi_J(x_t) \phi_J(x_{t+1})^T, \\
%    \Phi_J^T Q r &\approx \frac{1}{N+1}\sum_{k=0}^{N}\sum_{t=\tau^k}^{\tau^{k+1}-1} \phi_J(x_t) r(x_t), \\
%    \Phi_M^T Q \Phi_M &\approx \frac{1}{N+1}\sum_{k=0}^{N}\sum_{t=\tau^k}^{\tau^{k+1}-1} \phi_M(x_t) \phi_M(x_t)^T, \\
%    \Phi_M^T QP' \Phi_M &\approx \frac{1}{N+1}\sum_{k=0}^{N}\sum_{t=\tau^k}^{\tau^{k+1}-1} \phi_M(x_t) \phi_M(x_{t+1})^T, \\
%    \Phi_M^T QRr &\approx \frac{1}{N+1}\sum_{k=0}^{N}\sum_{t=\tau^k}^{\tau^{k+1}-1} \phi_M(x_t) r(x_t)^2, \\
%    \Phi_M^T QRP' \Phi_J &\approx \frac{1}{N+1}\sum_{k=0}^{N}\sum_{t=\tau^k}^{\tau^{k+1}-1} \phi_M(x_t) r(x_t) \phi_J(x_{t+1})^T .
%    \end{split}
%\end{equation}

%%%%%%%%%%%%%%%%%%%%%%%%%%%%%%%%%%%%%%%%%%%%%%%%%%%%%%%%%%%%%%%%%%%%%%%%
\section{Proof of Theorem \ref{th:TD_converges}\label{supp:th:TD_converges}}
%%%%%%%%%%%%%%%%%%%%%%%%%%%%%%%%%%%%%%%%%%%%%%%%%%%%%%%%%%%%%%%%%%%%%%%%
\begin{proof}
Using \eqref{eq:LSTD_proof_1} and \eqref{eq:LSTD_proof_2} we have for all $k$
\begin{equation}\label{eq:TD_proof_1}
    \begin{split}
      \mathbb{E}\left[\sum_{t=0}^{\tau^{k}-1} \phi_J(x_t)\delta_{J}^k(t,w_J,w_M)\right] &= \Phi_J^T Q r - \Phi_J^T Q \left( I - P \right) \Phi_J w_{J}, \\
      \mathbb{E}\left[\sum_{t=0}^{\tau^{k}-1} \phi_M(x_t)\delta_{M}^k(t,w_J,w_M)\right] &= \Phi_M^T Q R \left( r + 2P \Phi_J w_{J} \right) - \Phi_M^T Q \left( I - P \right) \Phi_M w_{M},
    \end{split}
\end{equation}
Letting $\hat{w}_k = (\hat{w}_{J;k}, \hat{w}_{M;k})$ denote a concatenated weight vector in the joint space $\mathbb{R}^{S_J}\times\mathbb{R}^{S_M}$ we can write the TD algorithm in a stochastic approximation form as
\begin{equation}\label{eq:SA_iterates}
      \hat{w}_{k+1} = \hat{w}_{k} + \xi_k \left(z + M \hat{w}_{k} + \delta M_{k+1}\right),
\end{equation}
where
\begin{equation*}
    M = \left(
                 \begin{array}{cc}
                   \Phi_J^T Q \left( P - I \right) \Phi_J & 0 \\
                   2\Phi_M^T Q R P \Phi_J & \Phi_M^T Q \left( P - I \right) \Phi_M \\
                 \end{array}
               \right),
\end{equation*}
\begin{equation*}
    z = \left(
                 \begin{array}{cc}
                   \Phi_J^T Q r \\
                   \Phi_M^T Q R r \\
                 \end{array}
               \right),
\end{equation*}
and the noise terms $\delta M_{k+1}$ satisfy
\begin{equation*}
\E\left[\delta M_{k+1}| F_n\right] = 0,
\end{equation*}
where $F_n$ is the filtration $F_n = \sigma(\hat{w}_{m},\delta M_m,m\leq n)$, since different trajectories are independent.
%Also, we have
%\begin{equation*}
%\mathbb{E}\left[\delta M_{k+1}^T\delta M_{k+1}\right] \leq K_1 + K_2 \hat{w}_k^T \hat{w}_k,
%\end{equation*}
%for some suitable $K_1$,$K_2$, since both the features and rewards are bounded.

We first claim that the eigenvalues of $M$ have a negative real part. To see this, observe that $M$ is block triangular, and its eigenvalues are just the eigenvalues of $\Phi_J^T Q \left( P - I \right) \Phi_J$ and $\Phi_M^T Q \left( P - I \right) \Phi_M$. By Lemma 6.10 in \cite{BT96} these matrices are negative definite. It therefore follows (see \citealp{Ber2012DynamicProgramming} example 6.6) that their eigenvalues have a negative real part. Thus, the eigenvalues of $M$ have a negative real part.

Next, let $h(w) = Mw+z$, and observe that the following conditions hold.
\begin{condition}\label{cond:a1}
The map $h$ is Lipschitz.
\end{condition}
\begin{condition}\label{cond:a2}
The step sizes satisfy
\begin{equation*}
    \sum_{k=0}^{\infty}\xi_k = \infty, \quad \sum_{k=0}^{\infty}\xi_k^2 < \infty.
\end{equation*}
\end{condition}
\begin{condition}\label{cond:a3}
$\{\delta M_n\}$ is a martingale difference sequence, i.e., $\E \left[\delta M_{n+1} | F_n\right] = 0$.
\end{condition}
The next condition also holds
\begin{condition}\label{cond:a4}
The functions $h_c(w) \triangleq h(cw)/c, c \geq 1$ satisfy $h_c(w) \to h_\infty(w)$ as $c \to \infty$, uniformly on compacts, and $h_\infty(w)$ is continuous. Furthermore, the Ordinary Differential Equation (ODE)
\begin{equation*}
    \dot{w}(t) = h_\infty(w(t))
\end{equation*}
has the origin as its unique globally asymptotically stable equilibrium.
\end{condition}
This is easily verified by noting that $h(cw)/c = Mw + c^{-1} z$, and since $z$ is finite, $h_c(w)$ converges uniformly as $c \to \infty$ to $h_\infty(w) = Mw$. The stability of the origin is guaranteed since the eigenvalues of $M$ have a negative real part.

Theorem 7 in Chapter 3 of \cite{borkar2008sto} states that if A\ref{cond:a1} - A\ref{cond:a4} hold, the following condition holds
\begin{condition}\label{cond:a5}
The iterates of \eqref{eq:SA_iterates} remain bounded almost surely, i.e., $\sup_k \|\hat{w}_{k}\|<\infty, \textrm{ a.s.}$
\end{condition}

Finally, we use a standard stochastic approximation result that, given that the above conditions hold, relates the convergence of the iterates of \eqref{eq:SA_iterates} with the asymptotic behavior of the ODE
\begin{equation}\label{eq:SA_ODE}
    \dot{w}(t) = h(w(t)).
\end{equation}
Since the eigenvalues of $M$ have a negative real part, \eqref{eq:SA_ODE} has a unique globally asymptotically stable equilibrium point, which by \eqref{eq:fixed_point_matrix_form} is exactly $\hat{w}* = (\hat{w}^*_{J}, \hat{w}^*_{M})$.
Formally, by Theorem 2 in Chapter 2 of \cite{borkar2008sto} we have that if A\ref{cond:a1} - A\ref{cond:a3} and A\ref{cond:a5} hold, then $\hat{w}_{k} \to \hat{w}*$ as $k \to \infty$ with probability 1.
\end{proof}

\end{document}